\documentclass{article}

\usepackage{graphicx}
\usepackage{natbib}
\usepackage{natbib}
\usepackage{color}
\usepackage{xcolor}
\usepackage{diagbox}
\usepackage{mathrsfs}
\usepackage{amsthm}
\usepackage{threeparttable}
\usepackage{caption} 
\usepackage{setspace}
\usepackage{hyperref}
\hypersetup{colorlinks,citecolor={blue}}

\usepackage{amssymb,amsmath}

\newtheorem{theorem}{{Theorem}}[section]

\usepackage{booktabs}
\usepackage[noend]{algpseudocode}
\usepackage{algorithmicx,algorithm}
\newcommand{\bW}{\textbf{W}}
\usepackage{arxiv}

\usepackage[utf8]{inputenc} % allow utf-8 input
\usepackage[T1]{fontenc}    % use 8-bit T1 fonts
\usepackage{hyperref}       % hyperlinks
\usepackage{url}            % simple URL typesetting
\usepackage{booktabs}       % professional-quality tables
\usepackage{amsfonts}       % blackboard math symbols
\usepackage{nicefrac}       % compact symbols for 1/2, etc.
\usepackage{microtype}      % microtypography
\usepackage{lipsum}

\title{IMMIGRATE: A Margin-based Feature Selection Method with Interaction Terms}

\author{
  Ruzhang Zhao \\
  Department of Biostatistics\\
  Bloomberg School of Public Health\\
  Johns Hopkins University \\
  Baltimore, MD 21205, USA\\
  \texttt{rzhao@jhu.edu} \\
   \And
 Pengyu Hong* \\
  Department of Computer Science\\
  Brandeis University\\
Waltham, MA 02453, USA\\
  \texttt{hongpeng@brandeis.edu} \\
  \And
Jun S. Liu* \\
  Department of Statistics\\
  Harvard University\\
Cambridge, MA 02138, USA\\
  \texttt{jliu@stat.harvard.edu} \\
}

\begin{document}
\maketitle

\begin{abstract}
Traditional hypothesis-margin researches focus on obtaining large margins and feature selection. In this work, we show that the robustness of margins is also critical and can be measured using entropy. In addition, our approach provides clear mathematical formulations and explanations to uncover feature interactions, which is often lack in large hypothesis-margin based approaches. We design an algorithm, termed IMMIGRATE (Iterative max-min entropy margin-maximization with interaction terms), for training the weights associated with the interaction terms. IMMIGRATE simultaneously utilizes both local and global information and can be used as a base learner in Boosting. We evaluate IMMIGRATE in a wide range of tasks, in which it demonstrates exceptional robustness and achieves the state-of-the-art results with high interpretability.
\end{abstract}

% keywords can be removed
\keywords{Hypothesis-margin \and Feature selection \and Entropy \and IMMIGRATE}

\section{Introduction}
Feature selection is one of the most fundamental problems in machine learning and pattern recognition \citep{fukunaga2013introduction}. The Relief algorithm by \citet{kira1992practical} is one of the most successful feature selection algorithms. It can be interpreted as an online learning algorithm that solves a convex optimization problem with a hypothesis-margin-based cost function. Instead of deploying exhaustive or heuristic combinatorial searches, Relief decomposes a complex, global and nonlinear classification task into a simple and local one. Following the large hypothesis-margin principle for classification, Relief calculates the weights of features, which can be used for feature selection. 
Considering the binary classification in a set of samples $\mathcal{P}$ with two kinds of labels, the hypothesis-margin of an instance $\vec{x}$ is later formally defined in \citet{gilad2004margin} %TJ Layout - et al. should not be in italics
 as $\frac{1}{2}(\|\vec{x} - \operatorname{NM}(\vec{x})\| - \|\vec{x} - \operatorname{NH}(\vec{x})\|)$, where $\operatorname{NH}(\vec{x})$ denotes the ``nearest hit," i.e., the nearest sample to $\vec{x}$ with the same label, while $\operatorname{NM}(\vec{x})$ denotes the ``nearest miss", the nearest sample to $\vec{x}$ with the different label.
The large hypothesis-margin principle has motivated several successful extensions of the Relief algorithm. For example, 
ReliefF \citep{kononenko1994estimating} uses multiple nearest neighbors. Simba \citep{gilad2004margin} recalculates the nearest neighbors every time the feature weights are updated. \citet{yang2008novel} consider global information to improve Simba. I-RELIEF \citep{sun2006iterative} identifies the nearest hits and misses in a probabilistic manner, which forms a variation of hypothesis-margin. LFE \citep{sun2008relief} extends Relief from feature selection to feature extraction using local information. IM4E is proposed by \citet{bei2015maximizing} to balance margin-quantity maximization and margin-quality maximization. Both approaches in \citet{sun2008relief}, \citet{bei2015maximizing} use a variation of hypothesis-margin proposed in \citet{sun2006iterative}.
%Another appealing property of the Relief-based algorithms is that they allow the examination of the knowledge associated with the nearest neighbors, which may not be reflected by the individual features.
\par
The Relief-based algorithms indirectly consider feature interactions by normalizing the feature weights \citep{urbanowicz2018Relief}, which, however, cannot directly reflect natural effects of associations and hence results in poor understanding on how feature interacts.
%interpretability on the effects of feature interactions. 
For example, Relief and many of its extensions cannot tell whether a high weight of a certain feature is caused by its linear effect or its interaction with other features \citep{urbanowicz2018Relief}. Furthermore, these methods cannot directly reveal and measure the impact of the interaction terms on classification results.
%Similarly, these methods cannot clearly reveal the influence of interaction terms on classification. In particular, the degree of such influence cannot be measured. 
\par
To this end, we propose the {\textbf{I}terative \textbf{M}ax-\textbf{MI}n entropy mar\textbf{G}in-maximization with inte\textbf{RA}ction \textbf{TE}rms} %%Please check if the italic and bold are necessary.
algorithm (IMMIGRATE, henceforth). IMMIGRATE directly measures the influence of feature interactions and has the following characteristics. First, when defining hypothesis-margin, we introduce a new trainable quadratic-Manhattan measurement to capture interaction terms, which measures the interaction importance directly. Second, we take advantage of the margin stability by measuring the underlying entropy based on the distribution of instances. Third, we derive an iterative optimization algorithm to efficiently minimize the cost function. Fourth, we design a novel classification method that utilizes the learned quadratic-Manhattan measurement to predict the class of a new instance. Fifth, we design a more powerful approach (i.e., Boosted IMMIGRATE) by using IMMIGRATE as the base learner of Boosting \citep{schapire1990strength}. Sixth, to make IMMIGRATE efficient for analyzing high-dimensional datasets, we take advantage of IM4E \citep{bei2015maximizing} to obtain an effective initialization.

\par
The rest of the paper is organized as follows. Section~2 explains the foundation of the Relief algorithm, and Section~3 introduces the IMMIGRATE algorithm. Section~4 summarizes and discusses our experiments with different datasets, showing that IMMIGRATE achieves the state-of-the-art results, and Boosted IMMIGRATE outperforms other boosting classifiers significantly. The computation time of IMMIGRATE is comparable to other popular feature selection methods that consider interaction terms. Section~5 concludes the article with comparisons with related works and a short discussion.
%%%%%%%%%%%%%%%%%%%%%%%%%%%%%%%%%%%%%%%%%%%%%%%%%%%%%%%%%%%%%%%%5

\section{Review: the Relief Algorithm}
We first introduce a few notations used throughout the paper: $\vec{x}_{i} \in \mathbb{R}^{A}$ as the $i$-th instance in the training set $\mathcal{P}$; $y_i$ as the class label of $\vec{x}_{i}$; $N$ as the size of $\mathcal{P}$; $A$ as the number of features (i.e., attributes); $\vec{w}$ as the feature weight vector; and $|\vec{x}_{i}|$ as a vector where absolute value operation is element-wise. Relief \citep{kira1992practical} iteratively calculates the feature weights in $\vec{w}$ (Algorithm \ref{ag1}). The higher a feature weight is, the more relevant the corresponding feature is. After the calculation of feature weights, a threshold is chosen to select relevant features. Relief can be viewed as a convex optimization problem that minimizes the cost function in Equation~\ref{eq:relief}:
\begin{equation}
\begin{aligned}
 C& =\sum_{n=1}^{M} \big( \vec{w}^{\,T}  \big|\vec{x}_{n} - \operatorname{NH}(\vec{x}_{n})\big| - \vec{w}^{\,T}  \big|\vec{x}_{n} - \operatorname{NM}(\vec{x}_{n})\big| \big), \\
& \text{subject to}:\vec{w}\geq 0, \,\|\vec{w}\|_{2}^{2} = 1,\, \\
\end{aligned}
\label{eq:relief}
\end{equation}
where $M( \ll N)$ is a user defined number of randomly chosen training samples, $\operatorname{NH}(\vec{x})$ is the nearest "hit" (from the same class) of $\vec{x}$; $\operatorname{NM}(\vec{x})$ is the nearest "miss" (from a different class) of $\vec{x}$; and $\vec{w}^{\,T}\big|\vec{x}_{n} - \operatorname{NH}(\vec{x}_{n})\big|$
is the weighted Manhattan distance. Denote $\vec{u} = \sum_{n=1}^{M} \big( \big|\vec{x}_{n} - \operatorname{NH}(\vec{x}_{n})\big| -  \big|\vec{x}_{n} - \operatorname{NM}(\vec{x}_{n})\big| \big)$. Minimizing the cost function of Relief (\ref{eq:relief}) can be solved using the Lagrange multiplier method and the Karush--Kuhn--Tucker conditions \citep{kuhn2014nonlinear} to get a closed-form solution: $\vec{w} = (-\vec{u})^{+}/\|  (-\vec{u})^{+} \|_{2}$,
where $(\vec{a})^{+}$ truncates the negative elements to 0. 
This solution to the original Relief algorithm 
is important for understanding the Relief-based algorithms.
\par

\begin{algorithm}[H]
\setstretch{1.25}
  \caption{The Original Relief Algorithm}\label{original}
$N$: the number of training instances.\\
$A$: the number of features(i.e. attributes).\\
$M$: the number of randomly chosen training samples to update feature weight $\vec{w}$.\\
{\bf Input}: a training dataset $\{ z_{n} = (\vec{x}_{n},y_{n})\}_{n=1,\cdots,N}$.\\
{\bf Initialization}: Initialize all feature weights to 0: $\vec{w} = 0$.
  \begin{algorithmic}
  \setstretch{1.65}
      \For {$i$ = 1 \textbf{to} $M$}
        \State Randomly select an instance $\vec{x}_i$ and find its  $\operatorname{NH}(\vec{x}_i)$ and $\operatorname{NM}(\vec{x}_i)$.
        \State Update the feature weights by $\vec{w} = \vec{w} - (\vec{x}_i - \operatorname{NH}(\vec{x}_i))^2/M + (\vec{x}_i - \operatorname{NM}(\vec{x}_i))^2/M$, 
        \State where the square operation is element-wise.
      \EndFor
  \end{algorithmic}
\textbf{Return}: $ \vec{w}$.
\label{ag1}
\end{algorithm}

\section{IMMIGRATE Algorithm}
%IMMIGRATE stands for \underline{\textbf{I}}terative \underline{\textbf{M}}ax-\underline{\textbf{M}}\underline{\textbf{I}}n entropy mar\underline{\textbf{G}}in-maximization with inte\underline{\textbf{R}}\underline{\textbf{A}}ction \underline{\textbf{TE}}rms algorithm (IMMIGRATE, henceforth).
Without loss of generality, we establish the IMMIGRATE algorithm in a general binary classification setting. This formulation can be easily extended to handle multi-class classification problems. 
%Our implementation of IMMIGRATE is applicable to multiple classification tasks. 
Let the whole data set be $\mathcal{P} = \{ z_{n}\mid z_{n}=(\vec{x}_{n}, y_{n}), \vec{x}_{n} \in \mathbb{R}^{A}, y_{n}=\pm 1\}_{n=1}^{N}$; the hit index set of $\vec{x}_n$ be $\mathcal{H}_{n} = \{ j\mid  z_{j}\in \mathcal{P}, y_{j}=y_{n} \,\& \,j \neq n\}$, and the miss index set of $\vec{x}_{n}$ be $\mathcal{M}_{n} = \{j\mid z_{j}\in \mathcal{P}, y_{j} \neq y_{n}\}$.
%; $\alpha_{n,h}$ be the probability that $\vec{x}_{h}$ is the nearest hit of an instance $\vec{x}_{n}$; and $\beta_{n,m}$ be the probability that $\vec{x}_{m}$ is the nearest miss of $\vec{x}_{n}$. 

\subsection{Hypothesis-Margin}
Given a distance $d(\vec{x}_{i},\vec{x}_{j})$ between two instances,
$\vec{x}_{i}$ and $\vec{x}_{j}$, 
a hypothesis-margin \citep{gilad2004margin} is defined as 
$\rho_{n,h,m} = d(\vec{x}_{n},\vec{x}_{m}) - d(\vec{x}_{n},\vec{x}_{h})$, where $\vec{x}_{h}$ and $\vec{x}_{m}$ represent the nearest hit and nearest miss for instance $\vec{x}_{n}$, respectively. We adopt the probabilistic hypothesis-margin defined by \citet{sun2006iterative} as
\begin{equation}
\rho_{n} = \sum_{m\in\mathcal{M}_{n}}\beta_{n,m}d(\vec{x}_{n},\vec{x}_{m}) - \sum_{h\in\mathcal{H}_{n}}\alpha_{n,h}d(\vec{x}_{n},\vec{x}_{h}),
\label{def2}
\end{equation}
where $\alpha_{n,h} \geq 0$, $\beta_{n,m} \geq 0$,  $\sum_{h\in\mathcal{H}_{n}}\alpha_{n,h} =1 $, $\sum_{m\in\mathcal{M}_{n}}\beta_{n,m} = 1$, for $\forall\, n \,\in \,\{1,\cdots, N\}$. As in the above design, the hidden random variable $\alpha_{n,h}$ represents the probability that $\vec{x}_{h}$ is the nearest hit of instance $\vec{x}_{n}$, while $\beta_{n,m}$ indicates the probability that $\vec{x}_{m}$ is the nearest miss of instance $\vec{x}_{n}$. In the rest of the paper, for conciseness, we will use margin to indicate hypothesis-margin.
\par
\subsection{Entropy to Measure Margin Stability}
The distributions of hits and misses can be used to evaluate the stability of margins (i.e., margin quality). %That is 
A more stable margin can be obtained by considering the distributions of instances with the same or different labels with respect to the target instance. A margin is deemed stable if it will not be greatly reduced by changes to only a few neighbors of the target instance. Considering an instance $\vec{x}_{n}$, its probabilities $\{\alpha_{n,h}\}$ and $\{\beta_{n,m}\}$ represent the distributions of its hits and misses, respectively. We can use the {\it hit entropy} $E_{hit}(\vec{x}_n) = -\sum_{h\in \mathcal{H}_{n}} \alpha_{n,h} \log\alpha_{n,h}$ and {\it miss entropy} $E_{miss}(\vec{x}_n) = -\sum_{m\in \mathcal{M}_{n}} \beta_{n,m} \log\beta_{n,m}$ to evaluate the stability of $\vec{x}_n$'s margin. The following two scenarios help explain the intuition of using these entropy. Scenario A: all neighbors are distributed evenly around the target instance; scenario B: the neighbor distribution is highly uneven. An extreme example for scenario B is that one instance is quite close to the target and the rest are quite far away from the target. An easy experiment to test the stability is to discard one instance from the system and to check how it influences the margin. In scenario A, if the closest neighbor (no matter if it is hit or miss) is discarded, the margin changes only slightly because there are many other hits/misses evenly distributed around the target. In scenario B, if the closest neighbor is a miss, its removal can increase the margin significantly. On the contrary, if the closest neighbor is a hit, removing it can decrease the margin significantly. Intuitively speaking, hits prefer scenario A and misses favor scenario B.

%%However, in scenario B, the disappearance of the true nearest hit can increases the margin significantly.  On the other hand, the disappearance of the true nearest miss makes the other misses have larger probabilities to be the nearest miss ($\{\beta_{n,m}\}$), which results in the increase of margin in Equation~\ref{def2}. However, if scenario B works for hits, the margin will decrease accordingly when the true nearest hit disappears. Similarly, if scenario A works for misses, the even distribution will not contribute to the margin. In conclusion, hits prefer scenario A and misses scenario B. 
\par
Since scenarios A and B correspond to high and low entropy, respectively, the margin can benefit from a large hit entropy $E_{hit}$ (e.g., scenario A) and a low miss entropy $E_{miss}$ (e.g., scenario B). We can set up a framework to maximize the hit entropy and minimize the miss entropy, which is equivalent to make the margin in Equation~\ref{def2} the most stable. \citet{bei2015maximizing} use the term max-min entropy principle to describe the process that maximizes the hit entropy and minimize the loss entropy to maximize the margin quality. The process of stabilizing margin is an extension of the large margin principle.

\subsection{Quadratic-Manhattan Measurement}
We extend the margin in Equation~\ref{def2} by using a new quadratic-Manhattan measurement defined as:
%{\color{red} Comment: Don't use $m$ in this definition because $m$ is used for something else in the rest of the paper. Correct it through the whole manuscript.}
\begin{equation}
q(\vec{x}_{i},\vec{x}_{j}) = \big|\vec{x}_{i} - \vec{x}_{j}\big|^{\,T}\textbf{W}\big|\vec{x}_{i} - \vec{x}_{j}\big|,
\label{def4}
\end{equation}
where $\textbf{W}$ is a non-negative symmetric matrix (element-wise non-negative) with its Frobenius norm $\|\textbf{W}\|_{F}=1$. The quadratic-Manhattan measurement is a natural extension of the weight vector, and the distance defined in Equation~\ref{def4} is a natural extension of the weighted Manhattan distance in Equation~\ref{eq:relief}. Off-diagonal elements in $\textbf{W}$ capture feature interactions and diagonal elements in $\textbf{W}$ capture main effects. % Here, we explain the motivation 
To understand why quadratic-Manhattan measurement can capture the influence of interactions, we observe that the effect of element
$w_{a,b}$ ($a \neq b$) in $\textbf{W}$
%, the element in the $a$-th row and $b$-th column of $\textbf{W}$, reflects the influence of the interactions between two features $a$ and $b$. In details, according to the extension of quadratic form,
enters into (\ref{def4}) 
as the coefficient for the combination of the $a$-th and $b$-th elements in vector $\big|\vec{x}_{i} - \vec{x}_{j}\big|$.  In Relief-based algorithms, the  weighted Manhattan distance Equation~\ref{eq:relief} can be equivalently captured by the feature weight update equation in Algorithm~\ref{ag1}. Similarly, $w_{a,b}$ can be updated using the combination of the $a$-th and $b$-th features based on a randomly given instance.
%, which is a straightforward way to understand the process of capturing interactions.
We thus define our new margin using the quadratic-Manhattan measurement as
\begin{equation}
\sum_{m\in\mathcal{M}_{n}}\beta_{n,m}q(\vec{x}_{n},\vec{x}_{m}) - \sum_{h\in\mathcal{H}_{n}}\alpha_{n,h}q(\vec{x}_{n},\vec{x}_{h}).   \end{equation}

\subsection{IMMIGRATE}
We design the following cost function to maximize our new margin, and simultaneously, the hit entropy and miss entropy are optimized.

\begin{equation}
\begin{aligned}
 C & =  \sum_{n=1}^{N} \bigg(\sum_{h\in \mathcal{H}_{n} }  \alpha_{n,h}\big|\vec{x}_{n} - \vec{x}_{h}\big|^{\,T}\textbf{W}\big|\vec{x}_{n} - \vec{x}_{h}\big| -  \sum_{m\in \mathcal{M}_{n} }  \beta_{n,m}\big|\vec{x}_{n} - \vec{x}_{m}\big|^{\,T} \textbf{W} \big|\vec{x}_{n} - \vec{x}_{m}\big| \bigg) \\
& + \sigma \sum_{n=1}^{N} [E_{miss}(z_{n})-E_{hit}(z_{n})] , \\
& \text{subject to}:\textbf{W} \geq 0, \,\textbf{W}^{T} = \textbf{W},\, \|\textbf{W}\|_{F}^{2} = 1, \\
& \forall n, \sum_{h\in \mathcal{H}_{n}} \alpha_{n,h} = 1 ,  \sum_{m\in \mathcal{M}_{n}} \beta_{n,m} = 1, \text{and} \,\, \alpha_{n,h} \geq 0, \beta_{n,m} \geq \,0,\,\,  \\
\end{aligned}
\label{eq11}
\end{equation}
where 
$E_{miss}(z_{n}) =  - \sum_{m\in \mathcal{M}_{n} }  \beta_{n,m}\log\beta_{n,m}$, $E_{hit}(z_{n}) =  - \sum_{h\in \mathcal{H}_{n} }  \alpha_{n,h}\log\alpha_{n,h} $, and $\sigma$ is a hyperparameter that can be tuned via internal cross-validation. 
\par
We also design the following optimization procedure containing two iterative steps to find \textbf{W} that minimizes the cost function. The framework starts from a randomly initialized \textbf{W} and stops when the change of cost function is less than a preset limit or the iteration number reaches a preset threshold. In practice, we find that it typically takes less than 10 iterations to stop and obtain good results. Based on our experiments,  different initialization of \textbf{W} does not influence the results of the iterative optimization. %Our iterative optimization strategy is efficient to achieve reasonably good results. 
The computation time of IMMIGRATE is comparable to other interaction related methods such as SODA \citep{li2018robust}, hierNet \citep{bien2013lasso}.
\par

As depicted by the flow-chart in  Figure~\ref{fig:flowchart}, the IMMIGRATE algorithm iteratively optimizes the cost function Equation~\ref{eq11}. It starts with a random initiation satisfying certain boundary conditions, and proceeds to iterate the two steps as detailed below in Algorithm~\ref{ag5}.

\begin{figure}[H]
    \centering
    \includegraphics[width=10cm]{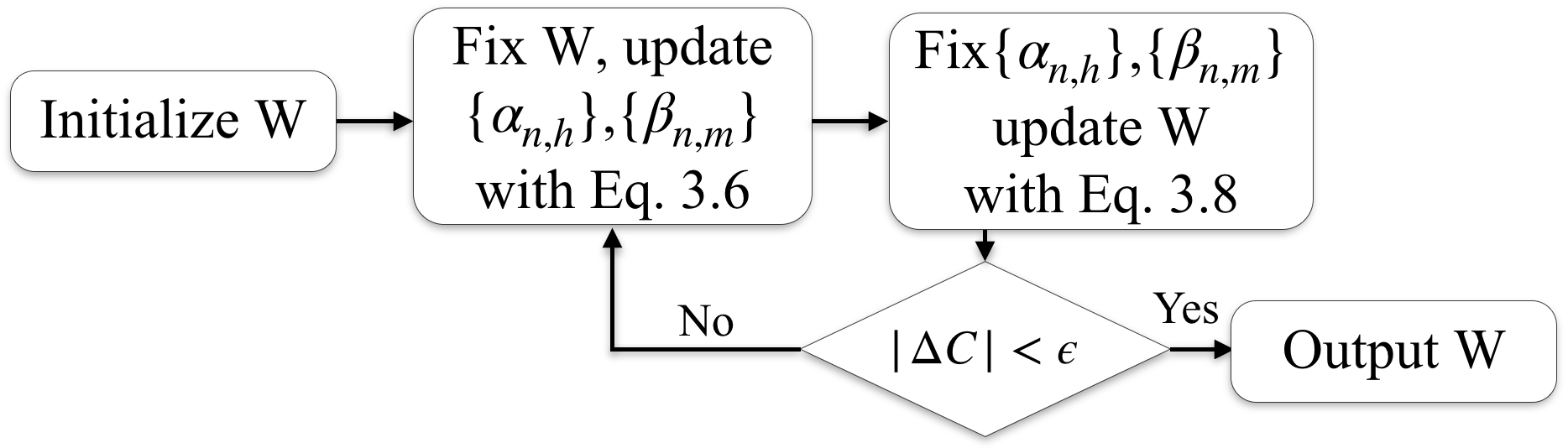}
    \caption{Flow chart of IMMIGRATE.  {\it Step 0}: Initialize $\textbf{W}$  randomly,  under the constraints $\textbf{W}\geq 0$, $\textbf{W}^{T} = \textbf{W}$ and $\|\textbf{W}\|_{F}^{2}=1$).   {\it Step 1}: Fix $\textbf{W}$, update $\{\alpha_{n,h}\}$ and $\{\beta_{n,m}\}$.  {\it Step 2}: Fix $\{\alpha_{n,h}\}$ and $\{\beta_{n,m}\}$, update $\textbf{W}$. Steps 1 and 2 are iterated to optimize the cost function, where $\Delta C$ is the change of the cost function in (\ref{eq11}) and $\epsilon$ is a pre-set limit. }
    \label{fig:flowchart}
\end{figure}

\subsubsection{Step 1: Fix $\textbf{W}$, Update $\{\alpha_{n,h}\}$ and $\{\beta_{n,m}\}$}
\label{sec:updatealpha}
Fixing $\textbf{W}$ and setting $\frac{\partial C}{\partial \alpha_{n,h}} = 0$ and $\frac{\partial C}{\partial \beta_{n,m}} = 0 $, we can obtain the closed-form updates of $\alpha_{n,h}$ and $\beta_{n,m}$ as
\begin{equation}
\begin{aligned}
& \alpha_{n,h} =\frac{exp(-q(\vec{x}_{n},\vec{x}_{h})/\sigma)}{\sum_{h\in\mathcal{H}_{n}}exp(-q(\vec{x}_{n},\vec{x}_{h})/\sigma)}, \\
& \beta_{n,m} = \frac{ exp(-q(\vec{x}_{n},\vec{x}_{m})/\sigma)}{\sum_{k\in\mathcal{M}_{n}}exp(-q(\vec{x}_{n},\vec{x}_{k})/\sigma)}.\\
\end{aligned}
\label{eq4}
\end{equation}
\par
The Hessian matrix of C w.r.t. probability pair ($\alpha_{n,h}$, $\beta_{n,m}$) is:
\begin{equation}
\begin{aligned}
\frac{\partial^{2}C}{\partial(\alpha_{n,h},\beta_{n,m})} = \left(\begin{array}{ccc} \sigma/\alpha_{n,h} & \partial^{2}C/\partial\beta_{n,m}\alpha_{n,h}  \\  \partial^{2}C/\partial\beta_{n,m}\alpha_{n,h} & - \sigma/\beta_{n,m}\end{array}\right).\\
\end{aligned}
\end{equation}
Since $\alpha_{n,h}, \beta_{n,m} > 0$, the determinant of the Hessian matrix is negative, where a saddle point is found in the $(\alpha_{n,h}, \beta_{n,m} )$ space. Therefore, the cost function $C$ achieves its local minimum and local maximum w.r.t. $\alpha_{n,h}$ and $\beta_{n,m}$, respectively.

\subsubsection{Step 2: Fix $\{\alpha_{n,h}\}$ and $\{\beta_{n,m}\}$, Update $\textbf{W}$}
\label{sec:updateW}
Fixing $\alpha_{n,h}$ and $\beta_{n,m}$, the minimization w.r.t. $\textbf{W}$ is convex. In Equation~\ref{eq11}, \textbf{W} satisfies $\textbf{W} \geq 0, \,\textbf{W}^{T} = \textbf{W},\, \|\textbf{W}\|_{F}^{2} = 1$. In our iterative optimization strategy, we impose \textbf{W} to be a distance metric for computation. Then, a closed-form solution to \textbf{W} can be derived (see Equation \ref{solve}).

\begin{theorem}
With $\{\alpha_{n,h}\}$ and $\{\beta_{n,m}\}$ fixed,  Equation~\ref{eq11} gives rise to  a closed-form solution for  updating $\textbf{W}$. Let
\begin{equation*}
\mathbf{\Sigma} = \sum_{n=1}^{N} \left(\Sigma_{n,H}-\Sigma_{n,M} \right), 
% \  \ \mathbf{\Sigma} \,  \psi_{i} = \mu_{i}\,\psi_{i}, \ \|\psi_{i}\|_{2}^{2} = 1, \ \mbox{for } \  i=1,\ldots, A,
%\label{eq13}
\end{equation*}
where $\Sigma_{n,H} =\sum_{h\in \mathcal{H}_{n} }  \alpha_{n,h}\big|\vec{x}_{n} - \vec{x}_{h}\big|\big|\vec{x}_{n} - \vec{x}_{h}\big|^{\,T} $, $\Sigma_{n,M} = \sum_{m\in \mathcal{M}_{n} } \beta_{n,m}\big|\vec{x}_{n} - \vec{x}_{m}\big|\big|\vec{x}_{n} - \vec{x}_{m}\big|^{\,T} $. Let the $\psi_{i}$'s  and $\mu_{i}$'s be the eigenvectors and eigenvalues of $\mathbf{\Sigma}$, respectively, so that
% $\|\psi_{i}\|_{2}^{2} = 1$,
%$\mu_{1} \leq \mu_{2} \leq \cdots \leq \mu_{A}$, and
$\mathbf{\Sigma}   \psi_{i} = \mu_{i} \psi_{i}$ with $ \|\psi_{i}\|_{2}^{2} = 1.$
  Then, 
\begin{equation}
\mathbf{W} = \Phi\,\Phi^{T},
\label{solve}
\end{equation}
where $\Phi = (\sqrt{\eta_{1}}\psi_{1},\sqrt{\eta_{2}}\psi_{2},\cdots,\sqrt{\eta_{A}}\psi_{A}) $, $\sqrt{\eta_{i}}=\sqrt{(-\mu_{i})^{+}/\sqrt{\sum_{i=1}^{A}((-\mu_{i})^{+})^{2}}}$.
\label{thm1}
\end{theorem}

\begin{proof}

Since \textbf{W} is a distance metric matrix, it is symmetric and positive-semidefinite. Let $\lambda_{1} \geq \lambda_{2} \geq \cdots \geq \lambda_{A} \geq 0$ be eigenvalues of \textbf{W}, then the 
eigen-decomposition of \textbf{W} is 
\begin{equation}
\begin{aligned}
\textbf{W}  & = P\Lambda P^{\,T} =P\Lambda^{1/2}\Lambda^{1/2}P^{\,T}, \\
& = [\sqrt{\lambda_{1}}\,\,p_{1},\cdots, \sqrt{\lambda_{A}}\,\,p_{A} ][\sqrt{\lambda_{1}}\,\,p_{1},\cdots, \sqrt{\lambda_{A}}\,\,p_{A} ]^{\, T} \equiv \Phi \Phi^T,
\end{aligned}
\end{equation}
where $P$ is an orthogonal matrix, and
%Thus, $\left< p_{i}, p_{j}\right> = 0 $. 
 $\Phi = [\phi_{1},\cdots,\phi_{A}] \equiv [\sqrt{\lambda_{1}}p_{1},\cdots,\sqrt{\lambda_{A}}p_{A}]$. Thus, $\left< \phi_{i},\phi_{j}\right> = 0 $.
The constraint $\|\textbf{W}\|_{F}^{2} = 1$ can be simplified as:
%since $\textbf{W}$ can be decomposed to be some orthogonal vectors,
\begin{equation}
\|\textbf{W}\|^{2}_{F} =  \sum_{i,j}w_{i,j}^{2} = \sum_{i}(\phi_{i}^{\,T}\phi_{i} )^{2} = 1.
\end{equation}

Let us rearrange Equation \ref{eq11} as:
\begin{equation}
\begin{aligned}
\sum_{h\in \mathcal{H}_{n} }  \alpha_{n,h}\big|\vec{x}_{n} - \vec{x}_{h}\big|^{\,T}\textbf{W}\big|\vec{x}_{n} - \vec{x}_{h}\big|&  \text{tr}(\textbf{W}\sum_{h\in \mathcal{H}_{n} }  \alpha_{n,h}\big|\vec{x}_{n} - \vec{x}_{h}\big|\big|\vec{x}_{n} - \vec{x}_{h}\big|^{\,T}), \\
\text{tr}(\textbf{W}\Sigma_{n,H})  = \text{tr}(\Sigma_{n,H} \sum_{i = 1}^{A}\phi_{i}\phi_{i}^{\,T}) & = \sum_{i=1}^{A}\phi_{i}^{\,T}\Sigma_{n,H}\,\,\phi_{i}.\\ 
\end{aligned}
\label{newcost1}
\end{equation}
Then, Equation \ref{eq11} can be further simplified as:
\begin{equation}
\begin{aligned}
C  & = \sum_{i=1}^{A} \phi_{i}^{\,T}\Sigma \,\, \phi_{i}, \\
& \text{subject to}: \|\textbf{W}\|^{2}_{F} =  \sum_{i}(\phi_{i}^{\,T}\phi_{i} )^{2} = 1,\left< \phi_{i},\phi_{j}\right> = 0,
\end{aligned}
\label{newcost2}
\end{equation}
where $\Sigma = \sum_{n=1}^{N}\Sigma_{n,H}-\Sigma_{n,M}$ and $\Sigma_{n,H} =\sum_{h\in \mathcal{H}_{n} }  \alpha_{n,h}\big|\vec{x}_{n} - \vec{x}_{h}\big|\big|\vec{x}_{n} - \vec{x}_{h}\big|^{\,T} $, $\Sigma_{n,M} =  \sum_{m\in \mathcal{M}_{n}} \beta_{n,m}\big|\vec{x}_{n} - \vec{x}_{m}\big|\big|\vec{x}_{n} - \vec{x}_{m}\big|^{\,T} $.
The orthogonality condition can be ignored  because this condition is required in the constraint. %has already been satisfied at the last step.
The Lagrangian for the optimization problem in Equation \ref{newcost2}
is easy to obtain:
\begin{equation}
L = \sum_{i=1}^{A} \phi_{i}^{\,T}\Sigma \,\, \phi_{i} + \lambda(\sum_{i=1}^{A}(\phi_{i}^{\,T}\phi_{i} )^{2} - 1).
\label{l}
\end{equation}
Differentiating $L$ with respect to $\phi_{i}$ yields:
\begin{equation}
\partial{L}/\partial{\phi_{i}} = 2\Sigma\phi_{i} + 4\lambda\phi_{i}^{\,T}\phi_{i}\phi_{i} = 0.
\label{derive}
\end{equation}
Denote $\phi_{i}/\|\phi_{i}\|_{2} := \psi_{i}$. From Equation \ref{derive},  we have
\begin{equation}
\Sigma \,\, \psi_{i} = \mu_{i}\,\, \psi_{i},
\label{de}
\end{equation}
where $\mu_{i} = - 2\lambda \|\phi_{i}\|_{2}^{2}  $. Thus, $\psi_{i}$  and $\mu_{i}$ are an eigenvector and eigenvalue of $\Sigma$, respectively. 
\par
Let $\phi_{i} = \sqrt{\eta_{i}}\psi_{i}$, $\eta_{i} \geq 0$. Thus, $C   = \sum_{i = 1}^{A} \sqrt{\eta_{i}}\psi_{i}^{\,T} \Sigma \sqrt{\eta_{i}}\psi_{i} = \sum_{i = 1}^{A}  \eta_{i}\mu_{i} \psi_{i}^{\,T}\psi_{i} = \sum_{i = 1}^{A} \eta_{i}\mu_{i}$, and $ \|\textbf{W}\|^{2}_{F} =  \sum_{i}(\sqrt{\eta_{i}}\psi_{i}^{\,T}\sqrt{\eta_{i}}\psi_{i} )^{2} =\sum_{i}(\eta_{i})^{2} =  1 $.
Then, Equation \ref{newcost2} can be simplified to be 
\begin{equation}
 C   = \sum_{i = 1}^{A}  \eta_{i}\mu_{i}, \,\, \text{subject to}: \,\, \sum_{i=1}^{A}(\eta_{i})^{2} =  1 , \eta_{i} \geq 0.
\label{z}
\end{equation}
\par
Note that  Equation \ref{z} is exactly the same as the original Relief Algorithm (Algorithm~\ref{ag1}):
\begin{equation}
\vec{\eta} = (-\vec{\mu})^{+}/\|  (-\vec{\mu})^{+} \|_{2},
\label{s1}
\end{equation}
where $(\vec{a})^{+} = [max(a_{1},0),max(a_{2},0),\cdots,max(a_{I},0)]$, and $\phi_{i} = \sqrt{\eta_{i}}\psi_{i}$. It is also easy to see that the updated $\textbf{W}$ is  a distance metric.
%\par
%Using $\Phi = [\phi_{1},\cdots,\phi_{A}] = [\sqrt{\lambda_{1}}p_{1},\cdots,\sqrt{\lambda_{A}}p_{A}]$, 
%\begin{equation}
%\mathbf{W} = \Phi\Phi^{T}.
%\end{equation}
%\par
%The orthogonality condition is achieved, because $\|\textbf{W}\|^{2}_{F} = \sum_{i}(\phi_{i}^{\,T}\phi_{i} )^{2} = 1$.
%In addition, since $\mathbf{W} = \Phi\Phi^{T}$, updated $\textbf{W}$ is also a distance metric.
\end{proof}
\par
\begin{algorithm}[H]
\setstretch{1.35}
\caption{The IMMIGRATE Algorithm}%%referred.
 {\bf Input}: a training dataset $\{ z_{n} = (\vec{x}_{n},y_{n})\}_{n=1,\cdots,N}$.\\ 
{\bf Initialization}: Let $t = 0$, randomly initialize $\textbf{W}^{(0)}$ satisfying $\textbf{W}^{(0)}\geq 0$, $\textbf{W}^{T} = \textbf{W}$, $\|\textbf{W}^{(0)}\|_{F}^{2}=1$.
\begin{algorithmic}
\setstretch{1.05}
   \Repeat\\
   \State Calculate $ \{ \alpha_{n,h}^{(t+1)}\}$, $\{ \beta_{n,m}^{(t+1)}\}$ with Equation~\ref{eq4}.\\
   \State Calculate $\textbf{W}^{(t+1)}$ with Theorem \ref{thm1}, Equation~\ref{solve}.\\
   \State $t = t+1$.\\
   \Until{the change of $C$ in Equation~\ref{eq11} is small enough or the iteration indicator $t$ reaches a preset limit}.
\end{algorithmic}
  {\bf Output}: $\textbf{W}^{(t)}$.
\label{ag5}
\end{algorithm}
\subsubsection{Weight Pruning}
\label{sec:re}
Some previous Relief-based algorithms offer options to remove weights lower than a preset threshold. IMMIGRATE offers a similar option to prune small weights: set small elements in $\textbf{W}$ to 0. By default, we use a threshold to prune small weights to 0, where $\textbf{W}$ should be normalized w.r.t. Frobenius norm after the pruning. 

\par
\subsubsection{Predict New Samples}
A prediction rule based on  the learned weight matrix \textbf{W} can be formulated as:
\begin{equation}
\begin{split}
&\hat{y}^{\prime} = \arg\min_{c} \sum_{y_{n}=c}\alpha_{n}^{c}(\vec{x}^{\,\prime})q(\vec{x}^{\,\prime},\vec{x}_{n}),\\ &\alpha_{n}^{c}(\vec{x}^{\,\prime}) = \frac{exp\big( -q( \vec{x}^{\,\prime},\vec{x}_{n})/\sigma\big)}{\sum_{y_{k}=c}exp\big( -q(\vec{x}^{\,\prime}, \vec{x}_{k})/\sigma\big)},
\end{split}
\label{eq9}
\end{equation}
where $z^{\prime} = (\vec{x}^{\,\prime}, y^{\,\prime})$ is a new instance, $c$ denotes the class and $\hat{y}^{\prime}$ is the predicted label. This prediction method assigns a new instance to a class that maximizes its hypothesis-margin using the learned weight matrix \textbf{W}, which makes it more stable than the $k$-NN method used in the traditional Relief-based algorithms. 

\subsection{IMMIGRATE in Ensemble Learning}
Boosting \citep{schapire1990strength, freund1996experiments, freund1999alternating} has been widely used to create ensemble learners that produce the state-of-the-art results in many tasks. Boosting combines a set of relatively weak base learners to create a much stronger learner. To use IMMIGRATE as the base classifier in the AdaBoost algorithm \citep{freund1996experiments}, we modify the cost function Equation~\ref{eq11} to include sample weights and use the modified version in the boosting iterations. We name the algorithm BIM, standing for \underline{\textbf{B}}oosted \underline{\textbf{IM}}MIGRATE (Refer to Equation~\ref{eq:bim} and Algorithm~\ref{ag:bim} for more details about BIM. BIM schedules the adjustment of the hyperparameter $\sigma$ in its boosting iterations. It starts with $\sigma$ being a predefined $\sigma_{max}$ and gradually reduces $\sigma$ by multiplying it with $(\sigma_{min}/\sigma_{max})^{1/T}$ at each interaction until reaching $\sigma_{min}$, where $T$ is a predefined maximum number of boosting iterations.

\begin{equation}
\begin{aligned}
 C & =  \sum_{n=1}^{N} D(\vec{x}_{n})\bigg(\sum_{h\in \mathcal{H}_{n} }  \alpha_{n,h}\big|\vec{x}_{n} - \vec{x}_{h}\big|^{\,T}\textbf{W}\big|\vec{x}_{n} - \vec{x}_{h}\big| -  \sum_{m\in \mathcal{M}_{n} }  \beta_{n,m}\big|\vec{x}_{n} - \vec{x}_{m}\big|^{\,T} \textbf{W} \big|\vec{x}_{n} - \vec{x}_{m}\big| \bigg) \\
& + \sigma \sum_{n=1}^{N} D(\vec{x}_{n})[E_{miss}(z_{n})-E_{hit}(z_{n})] , \\
& \text{subject to}:\textbf{W} \geq 0, \,\textbf{W}^{T} = \textbf{W},\, \|\textbf{W}\|_{F}^{2} = 1, \\
& \forall n, \sum_{h\in \mathcal{H}_{n}} \alpha_{n,h} = 1 ,  \sum_{m\in \mathcal{M}_{n}} \beta_{n,m} = 1, \text{and} \,\, \alpha_{n,h} \geq 0, \beta_{n,m} \geq \,0,\,\,  \\
\end{aligned}
\label{eq:bim}
\end{equation}
where $E_{miss}(z_{n}) =  - \sum_{m\in \mathcal{M}_{n} }  \beta_{n,m}\log\beta_{n,m}$, $E_{hit}(z_{n}) =  - \sum_{h\in \mathcal{H}_{n} }  \alpha_{n,h}\log\alpha_{n,h} $, $\sum_{n = 1}^{N}D(\vec{x}_{n}) = 1$, and $D(\vec{x}_{n}) \geq 0,\,\,\,\forall \,\,n$.
\begin{algorithm}[H]
\setstretch{1.15}
  \caption{The BIM Algorithm}\label{original}
  $T$: the number of classifiers for BIM.\\
 {\bf Input} :
a training dataset $\{ z_{n} = (\vec{x}_{n},y_{n})\}_{n=1,\cdots,N}$. \\
{\bf Initialization} :
for each $\vec{x}_{n}$, set $D_{1}(\vec{x}_{n}) = 1/N$.
  \begin{algorithmic}
  \setstretch{1.55}
      \For {$t$ := 1 \textbf{to} $T$}
         \State Limit max number of iteration of IMMIGRATE less than preset.
        \State Train weak IMMIGRATE classifier $h_{t}(x)$ using a chosen $\sigma_{t}$ and weights $D_{t}(x)$ by Equation~\ref{eq:bim}.
        \State Compute the error rate $\epsilon_{t}$ as $\epsilon_{t} = \sum_{i=1}^{N}D_{t}(x_{i})I[y_{i}\neq h_{t}(x_{i})]$.
        \If{  $\epsilon_{t} \geq 1/2$ or $\epsilon_{t} = 0$ }
         \State Discard $h_{t}$, $T = T - 1$ and continue .
         \EndIf  
        \State Set $\alpha_{t} = 0.5\times\log[(1-\epsilon_{t})/\epsilon_{t}] $.
        \State Update $D(x_{i})$: For each $x_{i}$, \\
        \hspace*{0.4in}$D_{t+1}(x_{i}) = D_{t}(x_{i})\exp(\alpha_{t}I[y_{i}\neq h_{t}(x_{i})])$.
        \State Normalize $D_{t+1}(x_{i})$, so that $\sum_{i=1}^{N}D_{t+1}(x_{i}) = 1$.
        \EndFor
          
  \end{algorithmic}
 {\bf Output}: $h_{final}(x) = \arg\max_{y\in\{0,1\}}\sum_{t:h_{t}(x) = y }\alpha_{t}$.
  \label{ag:bim}
\end{algorithm}
\subsection{IMMIGRATE for High-Dimensional Data Space}
When applied to high-dimensional data, IMMIGRATE can incur a high computational cost because it considers the interactions between every feature pair. To reduce the computational cost, we first use IM4E \citep{bei2015maximizing} to learn a feature weight vector, which is used to initialize the diagonal elements of \textbf{W} in the proposed quadratic-Manhattan measurement. We also use the learned feature weight vector to help pre-screen the features, and keep only those with weights above a preset limit. In the remaining computation, we only model interactions between those chosen features. The discarded features after pre-screening  can be added back empirically based on the need of a specific application. We term this procedure IM4E-IMMIGRATE, which is effective and computationally efficient. It can also be boosted (Boosted IM4E-IMMIGRATE) to be stronger.

\section{Experiments}
In our experiments, all continuous features are normalized with mean zero and unit variance. And cross-validation is used here to compare the performances of various approaches. We have implemented IMMIGRATE in R and MATLAB. {The R package is available at} \href{https://CRAN.R-project.org/package=Immigrate}{https://CRAN.R-project.org/package=Immigrate}, {and the MATLAB version is available at} \href{https://github.com/RuzhangZhao/Immigrate-MATLAB-}{https://github.com/RuzhangZhao/Immigrate-MATLAB-}. Both IMMIGRATE and BIM can be accelerated by parallel computing as their computations are matrix-based.  

\subsection{Synthetic Dataset}
We first test the robustness of the IMMIGRATE algorithm using a synthesized dataset where we have two interacting features following Gaussian distributions in a binary classification setting. The simulated dataset contains 100 samples from one class governed by a Gaussian distribution with mean $(4,2)^{T}$ and the covariance matrix $\left(
\begin{array}{cc} 
1 & 0.5 \\
0.5 & 1 \\
\end{array}
\right)$ and another 100 samples from the other class governed by a Gaussian distribution with mean $(6,0)^{T}$ and the same covariance matrix. In addition, we add noises following a Gaussian distribution with mean $(8,-2)^{T}$ and the covariance  matrix $\left(
\begin{array}{cc} 
8 & 4 \\
4 & 8 \\
\end{array}
\right)$ to the fist class, and add noises following a Gaussian distribution with mean $(2,4)^{T}$ and the same covariance matrix to the second class. Figure~\ref{fig1} shows a scatter plot of the synthesized dataset containing 10\% samples from the noise distributions. The slope of the orange dotted line in Figure~\ref{fig1} is 1, which separates data with different labels. 
%From the view of classifiers such as Support Vector Machine (SVM), since the designed points extend in the perpendicular direction to the orange line, SVM tends to choose the orange line to separate samples with different labels.
\par
The noises are included to disturb the detection of the interaction term. The noise level starts from 5\%, and gradually increases by 5\% to 50\%. As the baseline, we apply logistic regression and observe that the $t$-test $p$-value of the interaction coefficient increases from $3\times 10^{-11}$ to $7\times 10^{-5}$ and  $0.7$ when the noise level increases  from 0\% to  10\% and 50\%. 
%by the logistic regression first to show why noises will affect the selection of interaction term. The significance $t$-test for the regression coefficients, in particular, for the coefficient of interaction, is conducted. The $p$-values for 0 noise, 10\% noise, 50\% noise are about $3\times 10^{-11}$, $7\times 10^{-5}$, $0.7$, respectively. Without noises, the interaction term is significant while with larger noises, the significance will be sharpened. 
Local Feature Extraction (LFE, \citet{sun2008relief}) is a Relief-based algorithm which considers interaction terms indirectly, though the interaction information is only used for feature extraction. We run IMMIGRATE and LFE on the synthesized datasets and compare the weights of the interaction term between features 1 and 2 in Figure~\ref{fig2}, which shows IMMIGRATE is more robust than LFE.

\begin{figure}[H]
\centering
\includegraphics[width=7.2cm]{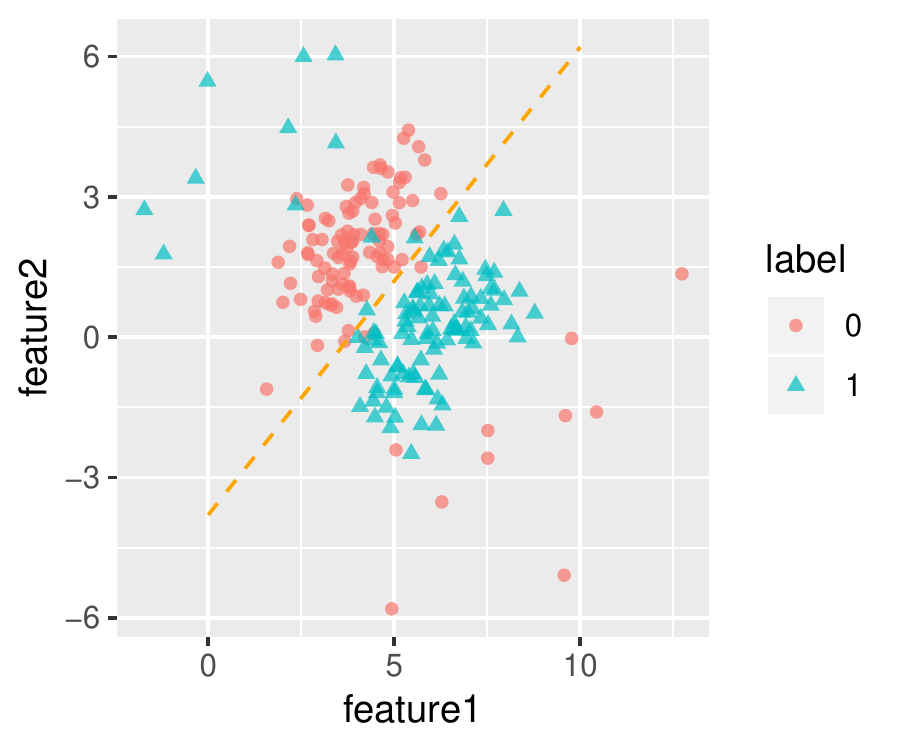}
\caption{The synthesized dataset with 10\% noise.}
\label{fig1}
\end{figure}

\begin{figure}[H]
\centering
\includegraphics[width=7cm]{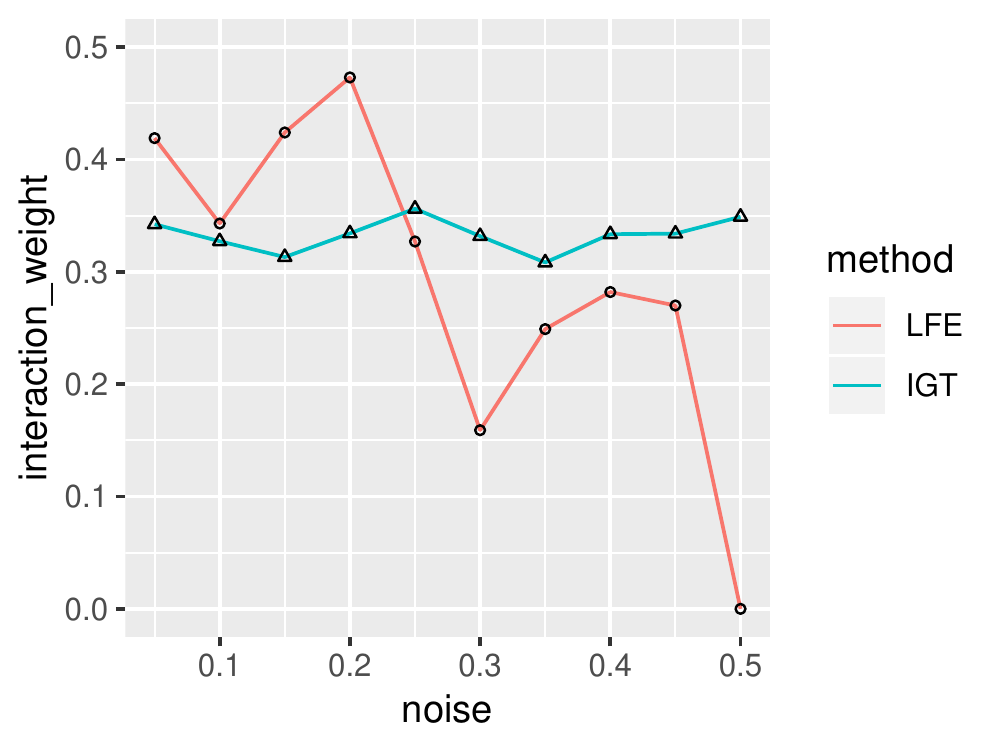}
\caption{IMMIGRATE (IGT) is more robust than LFE.}
\label{fig2}
\end{figure}

\subsection{Real Datasets}
We compare IMMIGRATE with several existing popular methods using real datasets from the \href{http://archive.ics.uci.edu/ml}{UCI database}. The following algorithms are considered in the comparison: Support Vector Machine \citep{soentpiet1999advances} with Sigmoid Kernel (SV1), Support Vector Machine with Radial basis function Kernel (SV2), LASSO (LAS) \citep{tibshirani1996regression}, Decision Tree (DT) \citep{freund1999alternating}, Naive Bayes Classifier (NBC) \citep{john1995estimating}, Radial basis function Network (RBF) \citep{haykin1994neural}, 1-Nearest Neighbor (1NN) \citep{aha1991instance}, 3-Nearest Neighbor (3NN), Large Margin Nearest Neighbor (LMN) \citep{ weinberger2009distance}, Relief (REL) \citep{kira1992practical}, ReliefF (RFF) \citep{kononenko1994estimating, robnik2003theoretical}, Simba (SIM) \citep{gilad2004margin}, and Linear Discriminant Analysis (LDA) \citep{fisher1936use}. In addition, several methods designed for detecting interaction terms are included: LFE \citep{sun2008relief}, Stepwise conditional likelihood variable selection for Discriminant Analysis (SOD) \citep{li2018robust}, and hierNet (HIN) \citep{bien2013lasso}. We also include three most widely used and competitive ensemble learners:  Adaptive Boosting (ADB) \citep{freund1996experiments, freund1999alternating}, Random Forest (RF) \citep{breiman2001random}, and XgBoost (XGB) \citep{chen2016xgboost}. We use the following abbreviations when presenting the results: IM4 for IM4E, IGT for IMMIGRATE, and B4G for the boosted IM4E-IMMIGRATE.

 Whenever possible, we use the settings of the aforementioned methods reported in their original papers: LMNN uses 3-NN classifier; Relief and Simba use Euclidean distance and 1-NN classifier; ReliefF uses Manhattan distance and $k$-NN classifier ($k$=1,3,5 is decided by internal cross-validation); in SODA, gam (=0, 0.5, 1) is determined by internal cross-validation and logistic regression is used for prediction. The IM4E algorithm has two hyperparameters $\lambda$ and $\sigma$. We fix $\lambda = 1$ as it has no actual contribution and tune $\sigma$ as suggested by \citet{bei2015maximizing}. Hence, the IMMIGRATE algorithm only has one hyperparameter $\sigma$. When tuning $\sigma$, we gradually decrease $\sigma$ from $\sigma_{0} = 4$ by half each time until it is not larger than $0.2$. The preset limit for weight pruning is $1/A$, where $A$ is the number of features. Furthermore, the preset iteration number is chosen to be 10. For each dataset, $\sigma$ and whether weight pruning is applied are determined by the best internal cross-validation results. For BIM, we use $\sigma_{max} = 4$, $\sigma_{min} = 0.2$, and the maximal number of boosting iterations $T$ is 100. The preset threshold in IM4E-IMMIGRATE is $2/A$.

We repeat ten-fold cross-validation ten times for each algorithm on each dataset, i.e., 100 trials are carried out. When comparing two algorithms (i.e., A vs. B), we calculate the paired Student's $t$-test using the results of 100 trials. First, the null hypothesis is there is no difference between the performances of A and those of B. When the $p$-value is larger than the significant level cutoff 0.05, we say A "Tie" B, which means there is no significant difference between their performances. When the $p$-value is smaller than the significant level cutoff 0.05, the second null hypothesis is the performances of B are no worse than those of A. When the new $p$-value is smaller than the significant level cutoff 0.05, we say A "wins", which means A on average performs significantly better than B on this dataset, and vice versa.

\subsubsection{Gene Expression Datasets}
Gene expression datasets typically have thousands of features. We use the following five gene expression datasets for feature selections: GLI \citep{freije2004gene}, Colon (COL) \citep{alon1999broad}, Myeloma (ELO) \citep{tian2003role}, Breast (BRE) \citep{van2002gene}, Prostate (PRO) \citep{singh2002gene}. All datasets have more than 10,000 features. Refer to Table~\ref{table:details} in Appendix A for details of all datasets. 
\par
We perform ten-fold cross-validation ten times, i.e., 100 trials in total. The results are summarized in Table~\ref{table:gene}. The last row "(W,T,L)" indicates the number of times that the Boosted IM4E-IMMIGRATE (B4G) W,T,L (win,tie,loss) compared with each algorithm by the paired Student's $t$-test with the significance level of $\alpha = 0.05$. The comparison results are also summarized in Figure~\ref{fig3} (top plot) for easy comparison. Although our B4G is not always the best, it outperforms other methods in most cases. In particular, when IM4E-IMMIGRATE (EGT) is compared with other methods, it also outperforms in most cases.

\begin{figure}[H]
\centering
\includegraphics[width=15cm]{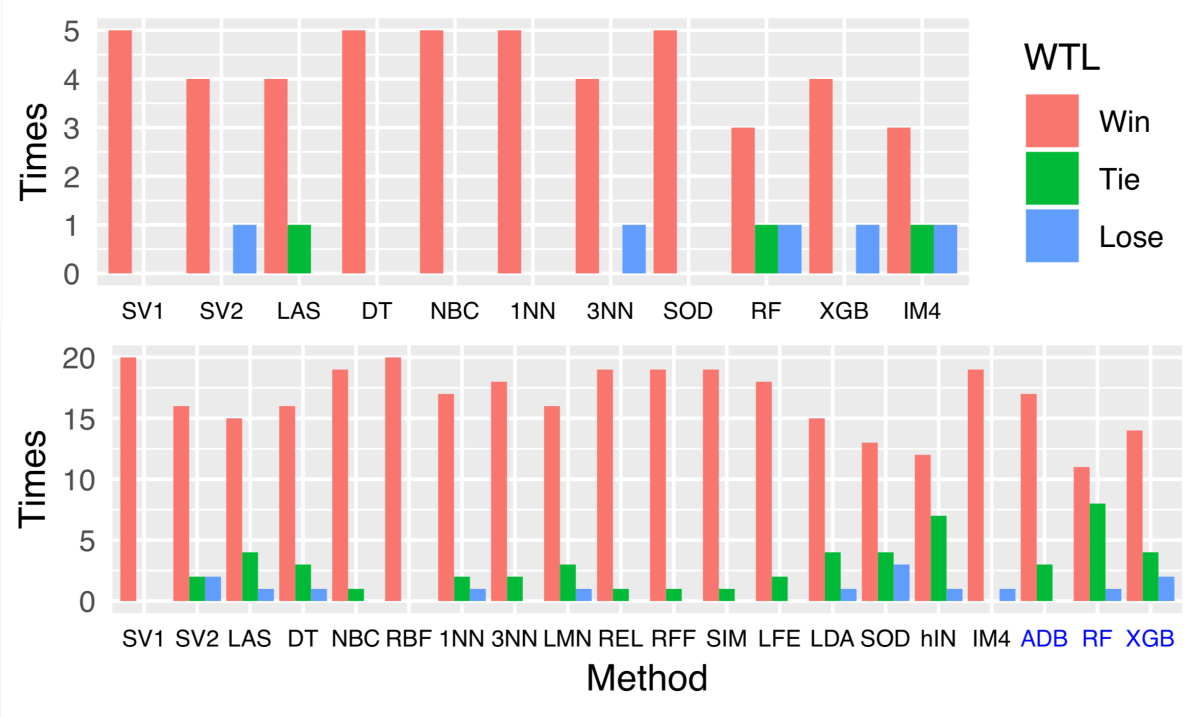}
\caption{Results of paired $t$-test on gene expression datasets (\textbf{top subplot}) and UCI datasets (\textbf{bottom subplot}). The top plot shows how well (i.e., "Win" (red bars), "Tie" (green bars), and "Lose" (blue bars)) our Boosted IM4E-IMMIGRATE performs compared with other approaches. In the bottom plot, the results of methods labeled in black are the comparisons with our IMMIGRATE, and the results of methods (ABD, RF, and XGB) labeled in blue are the comparisons with our BIM.}
\label{fig3}
\end{figure}

%The last row "(W,T,L)" indicates the number of times Boosted IM4E-IMMIGRATE (B4G) W,T,L (win,tie,loss) when compared with each algorithm by the paired Student's $t$-test with the significance level of $\alpha = 0.05$. 

% As shown in Figure~\ref{fig3} (Up) and Table~1 (supplemented), although Boosted IM4E-IMMIGRATE (B4G) is not always the best, it outperforms other methods in most cases. In particular, when comparing IM4E-IMMIGRATE (EGT) with other methods, it also outperforms in most cases.
\par

\begin{table}[H]
\begin{center}
\caption{Summarizes the accuracies on five high-dimensional gene expression datasets.} 
\vskip 0.03in
\label{table:gene}
\begin{threeparttable}
\setlength{\tabcolsep}{4pt}
\begin{tabular}{c|ccccccccccccccc}
\hline
Data&SV1&SV2&LAS&DT&NBC&1NN&3NN&SOD&RF&XGB&IM4&\textbf{EGT}&\textbf{B4G}\\
\hline
GLI&85.1&86.0&85.2&83.8&83.0&88.7&87.7&88.7&  87.6  &  86.3 & 87.5&89.1& 89.9 \\

COL&73.7&82.0&80.6&69.2&71.1&72.1&77.9&78.1&  82.6  &  79.5 & 84.3&78.6& 82.5 \\

ELO&72.9&90.2&74.6&77.3&76.3&85.6&91.3&86.9&79.2&77.9 & 88.9&88.6& 88.4\\ 

BRE& 76.0&88.7&91.4&76.4&69.4&83.0&73.6&82.6& 86.3 &  87.3 &  88.1&90.2& 91.5\\

PRO&71.3&69.9&87.9&86.4&68.0&83.2&82.7&83.2& 91.8& 90.5 &88.0&89.5&89.7\\

\hline
W,T,L$^{1}$&\textbf{\underline{5}},0,0
&\textbf{\underline{4}},0,1
&\textbf{\underline{4}},1,0
&\textbf{\underline{5}},0,0
&\textbf{\underline{5}},0,0
&\textbf{\underline{5}},0,0
&\textbf{\underline{4}},0,1
&\textbf{\underline{5}},0,0
&\textbf{\underline{3}},1,1
&\textbf{\underline{4}},0,1
&\textbf{\underline{3}},1,1
&-,-,-&-,-,-\\
\hline
\end{tabular}
\begin{tablenotes}
\item[1] \small{The last row shows the number of times Boosted IM4E-IMMIGRATE(\textbf{B4G}) W,T,L (win,tie,loss) compared with each algorithm by paired $t$-test\\

\item[**] Ten-fold cross-validation is performed for ten times, namely 100 trials are carried out for each dataset. The average accuracies are reported on the corresponding datasets in Table~1,2,3. Here, with 100 trials and two algorithms A and B, paired Student's $t$-test is carried out between the results of these two algorithms. Under the significance level of $\alpha = 0.05$, algorithm A is significantly better than another algorithm B (i.e. A wins) on a dataset if the $p$-value of the paired Student's $t$-test with corresponding null hypothesis is less than $\alpha = 0.05$. (The rule also applies to experiments on UCI datasets)
}. 
\end{tablenotes}
\end{threeparttable}
\end{center}
\end{table}

\subsubsection{UCI Datasets}
We also carry out an extensive comparison using many UCI datasets \citep{frank2010uci}: BCW, CRY, CUS, ECO, GLA, HMS, IMM, ION, LYM, MON, PAR, PID, SMR, STA, URB, USE and WIN. Refer to Appendix A Table~\ref{table:details} for the full names and links for those datasets. If a dataset has more than two classes, we use two classes with the largest sample size. In addition, we use three large-scale datasets: CRO$^{*}$, ELE$^{*}$, WAV$^{*}$.
\par
We perform ten-fold cross-validation ten times. Tables~\ref{table:uci} for IMMIGRATE and Table~\ref{table:uci2} for BIM show the average accuracies on the corresponding datasets. In Table~\ref{table:uci}, the last row "(W,T,L)" indicates the number of times IMMIGRATE (IGT) and BIM W,T,L (win,tie,loss) when compared with each algorithm separately by using the paired Student's $t$-test with the significance level of $\alpha = 0.05$. The comparison results are also summarized in Figure~\ref{fig3} (bottom subplot), where the first 17 items (black) indicate the results for IMMIGRATE while the last three items (blue) indicate the results for BIM.
\par
Although IMMIGRATE or BIM is not always the best, they outperform other methods significantly in one-to-one comparisons in terms of cross-validation results. Figure~\ref{fig3} (bottom subplot, black part) and Table~\ref{table:uci} show that IMMIGRATE achieves the state-of-the-art performance as the base classifier while Figure~\ref{fig3} (bottom subplot, blue part) and Table~\ref{table:uci2} show BIM achieves the state-of-the-art performance as the boosted version. To visualize the feature selection results of our approaches, we plot the feature weight heat maps of four datasets (GLA, LYM, SMR and STA) in Appendix~\ref{sec:heat} Figure~\hyperlink{afig1}{B1}.

\begin{table}[H]
\begin{center}
\addtolength{\leftskip} {-2cm}
    \addtolength{\rightskip}{-1cm}
\caption{Summarizes the accuracies on UCI datasets.} 
\label{table:uci}
\vskip 0.1in
\begin{threeparttable}
\setlength{\tabcolsep}{1pt}
\setlength{\tabbingsep}{1pt}
{
\renewcommand{\arraystretch}{0.95}
\begin{tabular}{c|cccccccccccccccccccccccc}
\hline
\textbf{Data}& SV1& SV2& LAS& 
DT& 
NBC& RBF& 
1NN& 3NN& LMN&REL& RFF& SIM& LFE& LDA& SOD & hIN & IM4& \textbf{IGT}\\
\hline
BCW&61.4& 66.6 &71.4& 70.5& 62.4& 56.9& 68.2& 72.2& 69.5&66.4 & 67.1 &67.7 &67.1 &73.9 &65.2 &71.8& 66.4&74.5 \\

CRY& 72.9& 90.6& 87.4& 85.3& 84.4& 89.7& 89.1& 85.4&87.8& 73.8& 77.2& 79.7& 86.0& 88.6& 86.0& 87.9& 86.2& 89.8\\

CUS & 86.5& 88.9& 89.6& 89.6& 89.5& 86.8& 86.5& 88.7&88.8&82.1& 84.7& 84.3& 86.4& 90.3& 90.8& 90.3& 87.5& 90.1 \\

ECO&92.9 &96.9 &98.6 &98.6 &97.8 &94.6 &96.0 &97.8 &  97.8&89.0 &90.7 &91.2 &93.1 &99.0 &97.9 &98.7 &97.5 & 98.2 \\

GLA&64.2&76.7&72.3&79.4&69.5&73.0&81.1&78.1&79.4&64.1&63.5&67.1&81.2&72.0&75.3&75.0&78.0& 87.5\\

HMS& 63.8& 64.5& 67.7& 72.5& 67.2& 66.8& 66.0& 69.3& 71.2 & 65.3& 66.0& 65.7& 64.9& 69.0& 67.4& 69.4& 66.6& 69.2\\

IMM& 74.3& 70.6& 74.4& 84.1& 77.9& 67.3& 69.4& 77.9& 76.7  & 69.9& 71.8& 69.0& 75.0& 75.2&72.3& 70.2& 80.7& 83.8\\

ION&80.5& 93.5 &83.6 &87.4 &89.4 &79.9 &86.7 &84.1 & 84.5&85.8 &86.2 &84.2 &91.0 &83.3 &90.3 &92.6 &88.3 & 92.9 \\

LYM&83.6&81.5&85.2&75.2&83.6&71.1&77.2&82.8& 86.6&64.9&71.0&70.4&
79.6&85.2&79.3&84.8&83.3&87.2\\

MON& 74.4& 91.7& 75.0& 86.4& 74.0& 68.2& 75.1& 84.4& 84.9 & 61.4& 61.8& 65.0& 64.8& 74.4& 91.9& 97.2& 75.6& 99.5 \\

PAR&72.7 &72.5 &77.1 &84.8 &74.1 &71.5 &94.6 &91.4 &91.8 &87.3 &90.3 &84.6 &94.0 &85.6 &88.2 &89.5 &83.2  & 93.8 \\

PID&65.6&73.1&74.7&74.3&71.2&70.3&70.3&73.5& 74.0&64.8&68.0&67.0&67.8&74.5&75.7&74.1&72.1& 74.7 \\

SMR& 73.5& 83.9& 73.6& 72.3& 70.3& 67.1& 86.9& 84.7&86.1& 69.5& 78.3& 81.0& 84.3& 73.1& 70.5& 83.0& 76.4& 86.5  \\ 

STA&69.8&71.6&70.8&68.9&71.0&69.5&67.8&70.8& 71.3&59.7&64.0&63.0&66.7&71.3& 71.8&69.2&70.8& 75.9\\

URB& 85.2& 87.9& 88.1& 82.6 & 85.8 & 75.3& 87.2 & 87.5 & 87.9& 81.9 & 83.2 & 73.0 & 87.9 & 73.0 & 87.9 & 88.3 & 87.4 & 89.9 \\

USE&95.7 &95.2 &97.2 &93.2 &90.6 &84.9 &90.5 &91.5 & 92.0&54.5 &63.7 &69.5 &85.8&96.9&96.2&96.5&94.1&96.4 \\

WIN& 98.3& 99.3& 98.6& 93.1& 97.3& 97.2& 96.4& 96.6& 96.5 & 87.2& 95.0& 95.0& 93.8& 99.7& 92.9& 98.9& 98.2& 99.0\\

CRO$^{*}$& 75.4 & 97.5 & 89.9 & 91.0 & 88.8 & 75.4 & 98.4 & 98.5 & 98.6 & 98.5 & 98.7 & 95.1 & 98.6 & 89.1 & 95.2 & 95.5 & 81.9 & 98.2 \\

ELE$^{*}$ & 72.3 & 95.7 & 79.9 & 80.0 &  82.5 & 70.8 & 81.1 & 83.9 & 89.7 & 64.6 & 75.4 & 76.2 & 79.8 & 79.9 & 93.7 & 93.6 & 83.2 & 93.7 \\

WAV$^{*}$ & 90.0 & 91.9 & 92.2 & 86.2 & 91.4 & 84.0 & 86.5 & 88.3 & 88.8 & 77.6 & 80.0 & 83.6 & 84.7 & 91.8 & 92.0 & 92.1 & 91.1 & 92.4 \\

\hline
W,T,L$^{1}$& \small{\textbf{\underline{20}}},0,0
&\small{\textbf{\underline{16}}},2,2
&\small{\textbf{\underline{15}}},4,1
&\small{\textbf{\underline{16}}},3,1
&\small{\textbf{\underline{19}}},1,0
&\small{\textbf{\underline{20}}},0,0
&\small{\textbf{\underline{17}}},2,1
&\small{\textbf{\underline{18}}},2,0
&\small{\textbf{\underline{16}}},3,1
&\small{\textbf{\underline{19}}},1,0
&\small{\textbf{\underline{19}}},1,0
&\small{\textbf{\underline{19}}},1,0
&\small{\textbf{\underline{18}}},2,0
&\small{\textbf{\underline{15}}},4,1
&\small{\textbf{\underline{13}}},4,3
&\small{\textbf{\underline{12}}},7,1
&\small{\textbf{\underline{19}}},0,1
&\small{-,-,-}\\
\hline
\end{tabular}}
\begin{tablenotes}
\item[1] \small{The last row (W,T,L) shows the number of times that IMMIGRATE (\textbf{IGT}) wins/ties/losses an existing algorithm according to the paired $t$-test on the cross-validation results.
}
\end{tablenotes}
\end{threeparttable}
\end{center}
\end{table}

\par

\begin{table}[H]
\begin{center}
\caption{Summarizes the accuracies on the UCI datasets.} 
\label{table:uci2}
\begin{threeparttable}
\setlength{\tabcolsep}{1pt}
\setlength{\tabbingsep}{1pt}
{\renewcommand{\arraystretch}{.95}
\begin{tabular}{c|cccccccccccccccccccccccc}
\hline
Data&ADB&RF&XGB& \textbf{BIM} \\
\hline

BCW&78.2 &78.6&78.6& 78.3\\

CRY&90.4&92.9 &89.9& 91.5\\

CUS & 90.8&91.1& 91.4& 91.0\\

ECO&98.0 &98.9& 98.2& 98.6 \\

GLA&85.0&87.0& 87.9& 86.8\\

HMS& 65.8& 72.1 &70.0&  72.0 \\

IMM&77.2&84.2& 81.7& 86.1 \\

ION&92.1&93.5& 92.5& 93.1 \\

LYM&84.8& 87.0& 87.4& 88.1 \\

MON& 98.4&95.8 &99.1&99.7\\
 
PAR & 90.5 &91.0 &91.9 &93.2\\

PID&73.5&76.0& 75.1& 76.2\\

SMR& 81.4& 82.8& 83.3&86.6\\

STA& 69.0& 71.3& 69.5 & 74.1\\

URB & 87.9 & 88.6 &88.8& 91.4 \\ 

USE&96.0&95.3& 94.9& 96.1 \\

WIN& 97.5& 99.1 &98.2& 99.1 \\

CRO$^{*}$& 97.3&  97.4&  98.5 & 98.6 \\

ELE$^{*}$ & 91.1 & 92.3 & 95.2 & 94.1 \\

WAV$^{*}$ & 89.5 & 91.2 & 90.8 & 93.3 \\

\hline
W,T,L$^{1}$&\textbf{\underline{17}},3,0&\textbf{\underline{11}},8,1&\textbf{\underline{14}},4,2& -,-,-\\
\hline
\end{tabular}}
\begin{tablenotes}
\item[1] \small{The last row (W,T,L) shows the number of times that the Boosted IMMIGRATE (\textbf{BIM}) wins/ties/losses an existing algorithm according to the paired $t$-test on the cross-validation results.} 
\end{tablenotes}
\end{threeparttable}
\end{center}
\end{table}

\section{Related Works}
In many recent publications, Relief-based algorithms and feature selection with interaction terms have been well explored. Some methods are reviewed here to show the connection and differences with our approach. The hypothesis-margin definition in Equation~\ref{def2} adopted in this work is also used in some previous studies, such as \citet{bei2015maximizing}. However, \citet{bei2015maximizing} do not consider the interactions between features. Our work provides a measurable way to show the influence of each feature interaction.
\par
\citet{sun2008relief} propose local feature extraction (LFE) method, which learns linear combination of features for feature extraction. LFE explores the information of feature interaction terms indirectly, which is partly our aim. However, LFE does not consider global information or margin stability, which results in significant differences in the cost function and the optimization procedures. 
\par
Our quadratic-Manhattan measurement defined in Equation~\ref{def4} is related to the Mahalanobis metric used in previous works on metric learning, such as Large Margin Nearest Neighbor (LMNN) \citep{weinberger2009distance}. \citet{weinberger2009distance} use semi-definite programming for learning distance metric in LMNN. LMNN and our approach are both based on K-Nearest Neighbors. A major difference is that our quadratic-Manhattan measurement has matrix $\textbf{W}$ to be non-negative and symmetric  (element-wise non-negative) with its Frobenius norm $\|\textbf{W}\|_{F}=1$, whereas
%While in metric learning, m
metric learning only requires its matrix to be symmetric semi-positive definite. Actually, the non-negative element requirement of $\bW$ provides IMMIGRATE a high intepretability, where items in matrix indicate interaction importance. Quadratic-Manhattan measurement serves well in the classification task and offers a direct explanation about how features, in particular, feature interaction terms, contribute to the classification results.

\section{Conclusions and Discussion}
In this paper, we propose a new quadratic-Manhattan measurement to extend the hypothesis-margin framework, based on which a  feature selection algorithm IMMIGRATE is developed for detecting and weighting interaction terms. We also develop its extended versions,  Boosted IMMIGRATE (BIM) and IM4E-IMMIGRATE.  IMMIGRATE and its variants follow the principle of maximizing stable hypothesis-margin and are implemented via a computationally efficient iterative optimization procedure.
%is designed for implementing the IMMIGRATE algorithm and the closed-form update of parameters is derived in Theorem \ref{thm1}. 
Extensive experiments show that IMMIGRATE outperforms   state-of-the-art methods significantly, and its boosted version BIM outperforms other boosting-based approaches. % Its robustness is clearly demonstrated on synthetic dataset where we know the ground truth.
In conclusion, compared with other Relief-based algorithms, IMMIGRATE mainly has the following advantages: (1) both local and global information are considered; (2) interaction terms are used; (3) robust and less prone to noise; (4) easily boosted. The computation time of IMMIGRATE variants is comparable to other methods able to detect interaction terms.  
%With the rapid development of deep neural network and etc., many classification task can be improved largely in accuracies with great  ability. However, classification task with small datasets, for example, less than 100 instances, still require the development like the proposed IMMIGRATE.
\par
There are some limitations for IMMIGRATE and we discuss some directions of improving the algorithm accordingly. First, in Section~\ref{sec:re}, small weights are removed to obtain sparse solutions using some cutoffs directly, which is hard to do inference for the obtained weights. Penalty terms such as the
$l_{1}$- or $l_2$-penalty are usually applied to shrink and select important weights. We suggest that our cost function Equation~\ref{eq11} can be modified to include such a penalty term to replace the process of weight pruning in Section~\ref{sec:re}. Second, although IMMIGRATE is efficient, it still costs much time to compute data with large size. To further improve the computational efficiency of IMMIGRATE for large-scale datasets, we can improve training by using well selected prototypes \citep{garcia2012prototype}, which, as a subset of the original data,  are representative but with noisy and redundant samples removed. %And prototype can represent the original data well. 
Third, IMMIGRATE only considers pair-wise interactions between features. Interactions among multiple features can play important roles in real applications \citep{yu2019multivariate,vinh2016can}. Our work provides a basis for developing new algorithms to detect multi-feature interactions. For example, people can use tensor form to consider weights for multi-feature interactions. Fourth, although our iterative optimization procedure is efficient, it achieves {ad hoc} solutions with no guarantee of reaching the global optimum. %In particular, procedure~\ref{sec:updatealpha} and \ref{sec:updateW} are both ad hoc solutions. 
It remains an open challenge to develop better optimization algorithms. Finally, the selection of an appropriate $\sigma$ currently relies on internal cross-validation, which cannot uncover the underlying properties of $\sigma$. A better strategy may be developed by rigorously investigating the theoretical contributions of $\sigma$. 
\par

%%%%%%%%%%%%%%%%%%%%%%%%%%%%%%%%%%%%%%%%%%
\vspace{6pt} 

%%%%%%%%%%%%%%%%%%%%%%%%%%%%%%%%%%%%%%%%%%
%% optional
%\supplementary{The following are available online at \linksupplementary{s1}, Figure S1: title, Table S1: title, Video S1: title.}


%%%%%%%%%%%%%%%%%%%%%%%%%%%%%%%%%%%%%%%%%%

%%%%%%%%%%%%%%%%%%%%%%%%%%%%%%%%%%%%%%%%%%

%%%%%%%%%%%%%%%%%%%%%%%%%%%%%%%%%%%%%%%%%%

\section*{Appendix}
\appendix

\section{Information of the Real Datasets}

%%%%%%%%%%%%%%%%%%%%%%%%%%%%%%%%%%%%%%%%%%%%%%%%%%%%%%%%%%%%%%%%%

\begin{table}[H]
\begin{center}
%\parbox{.9\linewidth}{
\centering
\caption{Summary of the UCI datasets and the gene expression datasets.}
\label{table:details}
%\vskip 0.1in
\setlength{\tabcolsep}{2.5pt}
\begin{threeparttable}
{\renewcommand{\arraystretch}{.95}
\begin{tabular}{c|cc|c}
\hline
Data&No.F$^{1}$ & No.I$^{2}$ & Full Name\\
\hline
BCW&  9 &  116 & \href{https://archive.ics.uci.edu/ml/datasets/Breast+Cancer+Wisconsin+(Prognostic)}{Breast Cancer Wisconsin (Prognostic)}\\

CRY& 6 & 90 & \href{https://archive.ics.uci.edu/ml/datasets/Cryotherapy+Dataset+}{Cryotherapy}\\

CUS & 7 & 440 & \href{https://archive.ics.uci.edu/ml/datasets/Wholesale\%2Bcustomers}{Wholesale customers}  \\

ECO& 5 & 220 & \href{https://archive.ics.uci.edu/ml/datasets/ecoli}{Ecoli}\\

GLA& 9 & 146 & \href{https://archive.ics.uci.edu/ml/datasets/glass+identification}{Glass Identification}  \\

HMS&  3 & 306 & \href{https://archive.ics.uci.edu/ml/datasets/Haberman\%27s+Survival}{Haberman's Survival} \\

IMM& 7 & 90 &  \href{https://archive.ics.uci.edu/ml/datasets/Immunotherapy+Dataset}{Immunotherapy} \\

ION& 32 & 351 & \href{https://archive.ics.uci.edu/ml/datasets/ionosphere}{Ionosphere} \\

LYM& 16 & 142 & \href{https://archive.ics.uci.edu/ml/datasets/Lymphography}{Lymphograph} \\

MON& 6 & 432 & \href{https://archive.ics.uci.edu/ml/datasets/MONK's+Problems}{MONK's Problems} \\
 
PAR &22 & 194 & \href{https://archive.ics.uci.edu/ml/datasets/parkinsons}{Parkinsons}  \\

PID& 8 & 768 & \href{https://github.com/cran/mlbench/blob/master/data/PimaIndiansDiabetes.rda}{Pima-Indians-Diabetes} \\

SMR& 60&  208 &\href{https://archive.ics.uci.edu/ml/datasets/Connectionist+Bench+\%28Sonar\%2C+Mines+vs.+Rocks\%29}{Connectionist Bench (Sonar, Mines vs. Rocks)}  \\

STA& 12 & 256 & \href{http://archive.ics.uci.edu/ml/datasets/statlog+(heart)}{Statlog (Heart)} \\

URB& 147 & 238 & \href{https://archive.ics.uci.edu/ml/datasets/Urban+Land+Cover}{Urban Land Cover} \\ 

USE& 5 & 251 & \href{https://archive.ics.uci.edu/ml/datasets/User+Knowledge+Modeling#}{User Knowledge Modeling} \\

WIN& 13 & 130 & \href{https://archive.ics.uci.edu/ml/datasets/wine}{Wine} \\

CRO$^{*}$& 28 & 9003 & \href{https://archive.ics.uci.edu/ml/datasets/Crowdsourced+Mapping}{ Crowdsourced Mapping}  \\

ELE$^{*}$ & 12 & 10000 & \href{https://archive.ics.uci.edu/ml/datasets/Electrical+Grid+Stability+Simulated+Data+}{Electrical Grid Stability Simulated} \\

WAV$^{*}$ & 21 & 3304 &\href{https://archive.ics.uci.edu/ml/datasets/Waveform+Database+Generator+(Version+1)}{ Waveform Database Generator}\\

GLI& 22283 & 85 & Gliomas Strongly Predicts Survival\citep{freije2004gene}\\

COL& 2000 & 62 &  Tumor and Normal Colon Tissues\citep{alon1999broad}\\

ELO& 12625 & 173 & Myeloma\citep{tian2003role}\\

BRE& 24481 & 78 & Breast Cancer\citep{van2002gene}\\

PRO& 12600 & 136 & Clinical Prostate Cancer Behavior\citep{singh2002gene} \\

\hline
\end{tabular}}
\begin{tablenotes}
\item[1] No.F: Number of Features.
\vskip 0.05in
\item[2] No.I: Number of Instances.
\end{tablenotes}
\end{threeparttable}
\end{center}
\end{table}

\section{Heat Maps}

\label{sec:heat}
\captionsetup[figure]{labelformat=empty, justification=centering}
\begin{figure}[H]
\centering
\begin{minipage}[t]{0.48\textwidth}
\centering
\includegraphics[width=6.8cm]{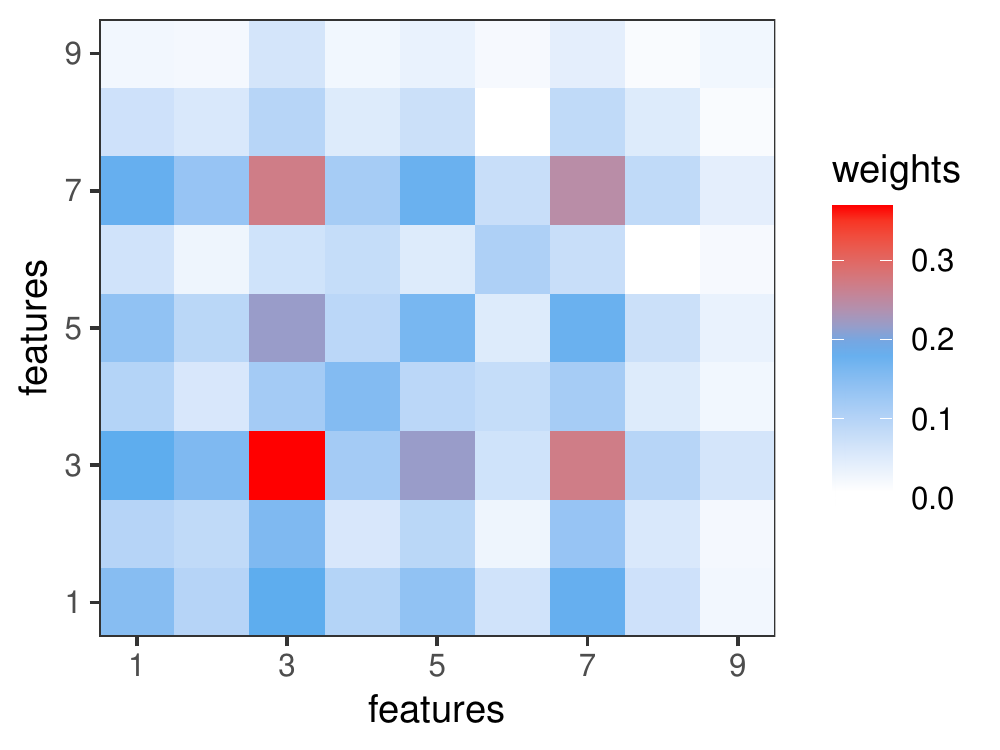}

\caption{GLA\hspace*{2em}}
\end{minipage}
\begin{minipage}[t]{0.48\textwidth}
\centering
\includegraphics[width=6.8cm]{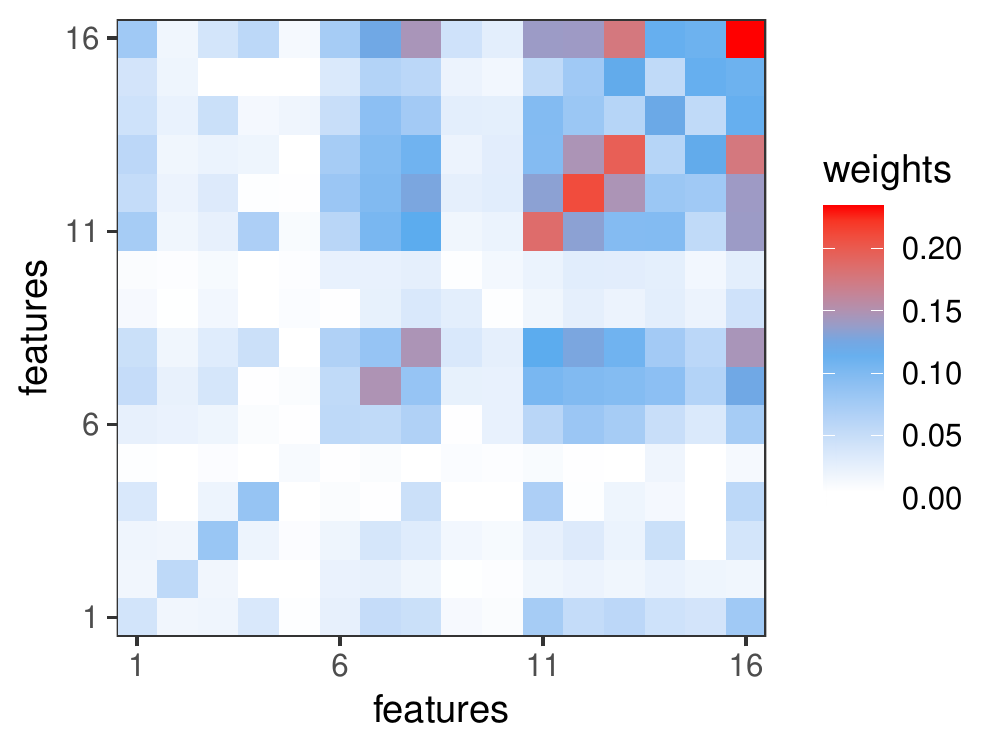}

\caption{LYM\hspace*{2em}}

\end{minipage}
\end{figure}

\vspace{0.8cm}
\begin{figure}[H]
\begin{minipage}[t]{0.48\textwidth}
\centering
\includegraphics[width=6.8cm]{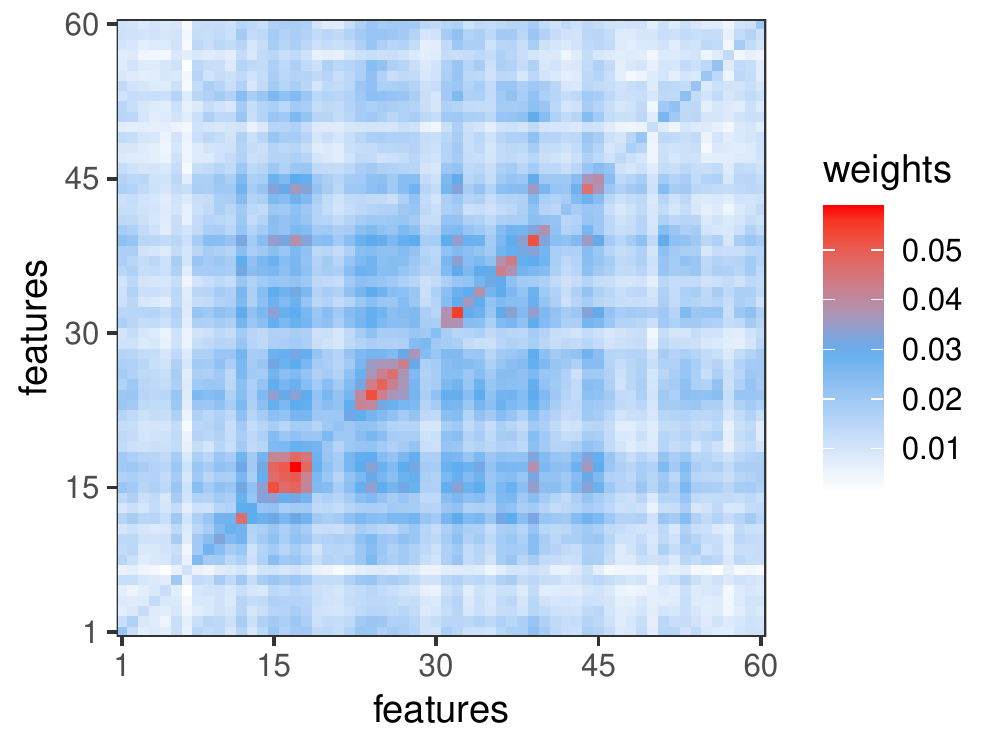}

\caption{SMR\hspace*{2em}}
\end{minipage}
\begin{minipage}[t]{0.48\textwidth}
\centering
\includegraphics[width=6.8cm]{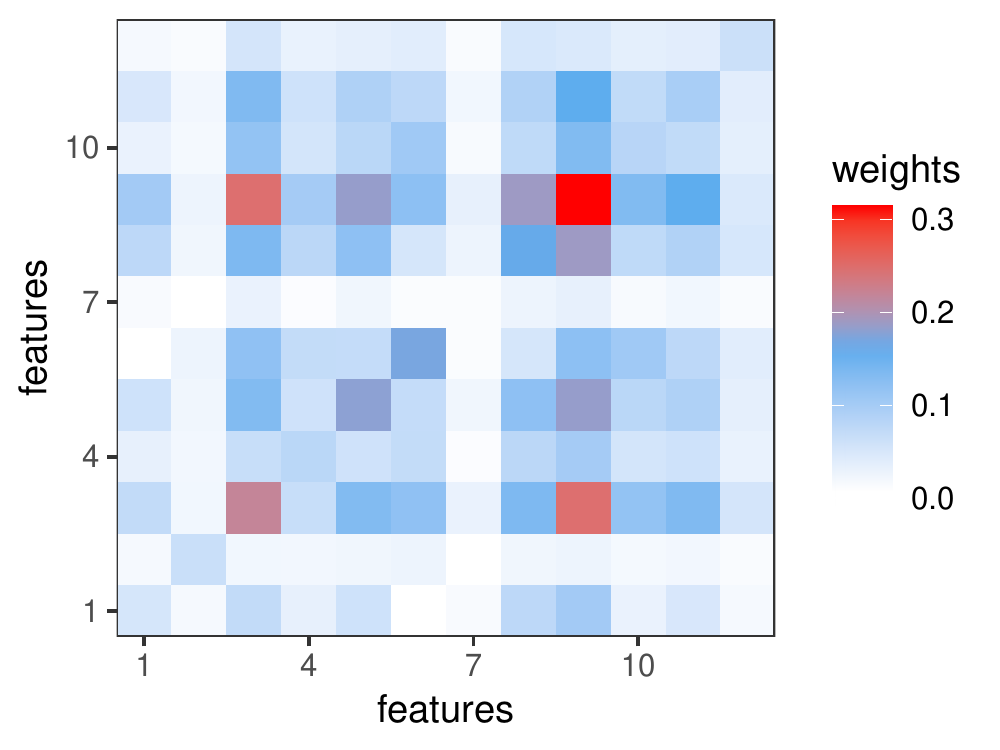}

\caption{STA\hspace*{2em}}
\end{minipage}
\begin{center}
\hypertarget{afig1}{}
\caption{\textbf{Figure}~\textbf{B1}. Heat Maps of Feature Weights Learned by IMMIGRATE. The color bars show the values of corresponding colors in the plots.}
\end{center}
\end{figure}

\bibliographystyle{unsrtnat}  
\bibliography{template}

\end{document}